\documentclass{article}

\usepackage{microtype}
\usepackage{graphicx}
\usepackage{subfigure}
\usepackage{booktabs} 
\usepackage[noend]{algpseudocode}
\usepackage{amsthm,amssymb,amsmath}
\usepackage{amsmath}
\usepackage{glossaries}
\usepackage{listings}

\usepackage{hyperref}

\newtheorem{theorem}{Theorem}
\newtheorem{corollary}{Corollary}
\newtheorem{lemma}{Lemma}

\DeclareMathOperator*{\argmin}{argmin}

\usepackage[accepted]{icml2019}

\icmltitlerunning{Deep Counterfactual Regret Minimization}

\begin{document}

\twocolumn[
\icmltitle{Deep Counterfactual Regret Minimization}



\icmlsetsymbol{equal}{*}

\begin{icmlauthorlist}
\icmlauthor{Noam Brown}{equal,cmu,fair}
\icmlauthor{Adam Lerer}{equal,fair}
\icmlauthor{Sam Gross}{fair}
\icmlauthor{Tuomas Sandholm}{cmu}
\end{icmlauthorlist}

\icmlaffiliation{fair}{Facebook AI Research}
\icmlaffiliation{cmu}{Computer Science Department, Carnegie Mellon University}

\icmlcorrespondingauthor{Noam Brown}{noamb@cs.cmu.edu}

\icmlkeywords{CFR, Counterfactual Regret Minimization, Deep Reinforcement Learning, Poker, Regret Minimization, No-Regret Learning, Deep Learning}

\vskip 0.3in
]



\printAffiliationsAndNotice{\icmlEqualContribution} 

\begin{abstract}
\vspace{-0.03in}
\emph{Counterfactual Regret Minimization (CFR)} is the leading framework for solving large imperfect-information games. 
It converges to an equilibrium by iteratively traversing the game tree.
In order to deal with extremely large games, abstraction is typically applied before running CFR. The abstracted game is solved with tabular CFR, and its solution is mapped back to the full game. This process can be problematic because aspects of abstraction are often manual and domain specific, abstraction algorithms may miss important strategic nuances of the game, and there is a chicken-and-egg problem because determining a good abstraction requires knowledge of the equilibrium of the game.
This paper introduces \emph{Deep Counterfactual Regret Minimization}, a form of CFR that obviates the need for abstraction by instead using deep neural networks to approximate the behavior of CFR in the full game. We show that Deep CFR is principled and achieves strong performance in large poker games. This is the first non-tabular variant of CFR to be successful in large games.
\end{abstract}

\section{Introduction}
\vspace{-0.03in}
\label{sec:intro}

Imperfect-information games model strategic interactions between multiple agents with only partial information. They are widely applicable to real-world domains such as negotiations, auctions, and cybersecurity interactions. Typically in such games, one wishes to find an approximate equilibrium in which no player can improve by deviating from the equilibrium.


The most successful family of algorithms for imperfect-information games have been variants of \emph{Counterfactual Regret Minimization (CFR)}~\cite{Zinkevich07:Regret}. CFR is an iterative algorithm that converges to a Nash equilibrium in two-player zero-sum games. Forms of tabular CFR have been used in all recent milestones in the benchmark domain of poker~\cite{Bowling15:Heads-up,Moravcik17:DeepStack,Brown17:Superhuman} and have been used in all competitive agents in the Annual Computer Poker Competition going back at least six years.\footnote{\href{www.computerpokercompetition.org}{www.computerpokercompetition.org}} In order to deal with extremely large imperfect-information games, \emph{abstraction} is typically used to simplify a game by bucketing similar states together and treating them identically. The simplified (abstracted) game is approximately solved via tabular CFR. However, constructing an effective abstraction requires extensive domain knowledge and the abstract solution may only be a coarse approximation of a true equilibrium.

In constrast, reinforcement learning has been successfully extended to large state spaces by using function approximation with deep neural networks rather than a tabular representation of the policy (deep RL). This approach has led to a number of recent breakthroughs in constructing strategies in large MDPs~\cite{Mnih15:Human} as well as in zero-sum perfect-information games such as Go~ \cite{Silver17:Mastering,Silver18:General}.\footnote{Deep RL has also been applied successfully to some partially observed games such as Doom~\cite{lample2017playing}, as long as the hidden information is not too strategically important.} Importantly, deep RL can learn good strategies with relatively little domain knowledge for the specific game~\cite{Silver17:Mastering}.
However, most popular RL algorithms do not converge to good policies (equilibria) in imperfect-information games in theory or in practice.

Rather than use tabular CFR with abstraction, this paper introduces a form of CFR, which we refer to as \emph{Deep Counterfactual Regret Minimization}, that uses function approximation with deep neural networks to approximate the behavior of tabular CFR on the full, unabstracted game. We prove that Deep CFR converges to an $\epsilon$-Nash equilibrium in two-player zero-sum games and empirically evaluate performance in poker variants, including heads-up limit Texas hold'em. We show Deep CFR outperforms Neural Fictitious Self Play (NFSP)~\cite{Heinrich16:Deep}, which was the prior leading function approximation algorithm for imperfect-information games, and that Deep CFR is competitive with domain-specific tabular abstraction techniques.


\section{Notation and Background}
\vspace{-0.03in}
\label{sec:background}
In an imperfect-information extensive-form (that is, tree-form) game there is a finite set of players, $\mathcal{P}$. A \emph{node} (or history) $h$ is defined by all information of the current situation, including private knowledge known to only one player. $A(h)$ denotes the actions available at a node and $P(h)$ is either chance or the unique player who acts at that node. If action $a \in A(h)$ leads from $h$ to $h'$, then we write $h \cdot a = h'$. We write $h \sqsubset h'$ if a sequence of actions leads from $h$ to $h'$. $H$ is the set of all nodes.
$Z \subseteq H$ are terminal nodes for which no actions are available. For each player $p \in \mathcal{P}$, there is a payoff function $u_p: Z\rightarrow \mathbb{R}$. 
In this paper we assume $\mathcal{P} = \{1, 2\}$ and $u_1 = -u_2$ (the game is two-player zero-sum).
We denote the range of payoffs in the game by $\Delta$.

Imperfect information is represented by {\em information sets} (infosets) for each player $p \in \mathcal{P}$. For any infoset $I$ belonging to $p$, all nodes $h, h' \in I$ are indistinguishable to $p$. Moreover, every non-terminal node $h \in H$ belongs to exactly one infoset for each $p$. We represent the set of all infosets belonging to $p$ where $p$ acts by $\mathcal{I}_p$. We call the set of all terminal nodes with a prefix in $I$ as $Z_I$, and we call the particular prefix $z[I]$. We assume the game features \emph{perfect recall}, which means if $h$ and $h'$ do not share a player $p$ infoset then all nodes following $h$ do not share a player $p$ infoset with any node following $h'$.

A strategy (or policy) $\sigma(I)$ is a probability vector over actions for acting player~$p$ in infoset $I$. Since all states in an infoset belonging to $p$ are indistinguishable, the strategies in each of them must be identical. The set of actions in $I$ is denoted by $A(I)$. The probability of a particular action $a$ is denoted by $\sigma(I,a)$.
We define $\sigma_p$ to be a strategy for $p$ in every infoset in the game where $p$ acts.
A strategy profile $\sigma$ is a tuple of strategies, one for each player. The strategy of every player other than $p$ is represented as $\sigma_{-p}$. $u_p(\sigma_p, \sigma_{-p})$ is the expected payoff for $p$ if player $p$ plays according to $\sigma_p$ and the other players play according to $\sigma_{-p}$.

$\pi^{\sigma}(h) = \Pi_{h' \cdot a \sqsubseteq h} \sigma_{P(h')}(h',a)$ is called {\em reach} and is the probability $h$ is reached if all players play according to $\sigma$. $\pi^{\sigma}_p(h)$ is the contribution of $p$ to this probability.
$\pi^{\sigma}_{-p}(h)$ is the contribution of chance and all players other than $p$.
For an infoset $I$ belonging to $p$, the probability of reaching $I$ if $p$ chooses actions leading toward $I$ but chance and all players other than $p$ play according to $\sigma_{-p}$ is denoted by $\pi^{\sigma}_{-p}(I) = \sum_{h \in I}\pi^{\sigma}_{-p}(h)$.
For $h \sqsubseteq z$, define $\pi^\sigma(h \to z) = \Pi_{h' \cdot a \sqsubseteq z,h'\not\sqsubset h} \sigma_{P(h')}(h',a)$ 

A {\em best response} to $\sigma_{-p}$ is a player~$p$ strategy $BR(\sigma_{-p})$ such that $u_p\big(BR(\sigma_{-p}), \sigma_{-p}\big) = \max_{\sigma'_p} u_p(\sigma'_p, \sigma_{-p})$.
A {\em Nash equilibrium} $\sigma^*$ is a strategy profile where everyone plays a best response: $\forall p$, $u_p(\sigma^*_p, \sigma^*_{-p}) = \max_{\sigma'_p} u_p(\sigma'_p, \sigma^*_{-p})$~\cite{Nash50:Eq}.
The {\em exploitability} $e(\sigma_p)$ of a strategy $\sigma_p$ in a two-player zero-sum game is how much worse $\sigma_p$ does versus $BR(\sigma_p)$ compared to how a Nash equilibrium strategy $\sigma^*_p$ does against $BR(\sigma^*_p)$. Formally, $e(\sigma_p) = u_p\big(\sigma_p^*,BR(\sigma_p^*)\big) - u_p\big(\sigma_p,BR(\sigma_p)\big)$. We measure \emph{total exploitability} $\sum_{p \in P} e(\sigma_p)$\footnote{Some prior papers instead measure \emph{average} exploitability rather than \emph{total} (summed) exploitability.}.

\subsection{Counterfactual Regret Minimization (CFR)}
\vspace{-0.03in}
CFR is an iterative algorithm that converges to a Nash equilibrium in any finite two-player zero-sum game with a theoretical convergence bound of $O(\frac{1}{\sqrt{T}})$. In practice CFR converges much faster. We provide an overview of CFR below; for a full treatment, see~\citet{Zinkevich07:Regret}. Some recent forms of CFR converge in $O(\frac{1}{T^{0.75}})$ in self-play settings~\cite{Farina19:Stable}, but are slower in practice so we do not use them in this paper.


Let $\sigma^t$ be the strategy profile on iteration $t$.  The \emph{counterfactual value} $v^{\sigma}(I)$ of player $p=P(I)$ at $I$ is the expected payoff to $p$ when reaching $I$, weighted by the probability that $p$ would reached $I$ if she tried to do so that iteration. Formally,
\begin{equation}
\label{def:v}
v^\sigma(I) = \sum_{z\in Z_I} \pi_{-p}^\sigma(z[I])\pi^\sigma(z[I] \to z) u_p(z)
\end{equation}
and $v^\sigma(I,a)$ is the same except it assumes that player~$p$ plays action~$a$ at infoset~$I$ with 100\% probability.

The {\em instantaneous regret} $r^t(I,a)$ is the difference between $P(I)$'s counterfactual value from playing $a$ vs. playing $\sigma$ on iteration $t$
\begin{equation}
r^t(I,a) = v^{\sigma^t}(I,a) - v^{\sigma^t}(I)
\label{eq:instant}
\end{equation}
The {\em counterfactual regret} for infoset $I$ action $a$ on iteration $T$ is
\begin{equation}
R^T(I,a) = \sum_{t = 1}^T r^t(I,a)
\label{eq:regret}
\end{equation}
Additionally, $R^T_+(I,a) = \max\{R^T(I,a), 0 \}$ and $R^T(I) = \max_a\{R^T(I,a)\}$. \emph{Total regret}
for $p$ in the entire game is
$R_p^T = \max_{\sigma_p'} \sum_{t = 1}^T \big(u_p(\sigma'_p, \sigma_{-p}^t) - u_p(\sigma^t_p, \sigma_{-p}^t)\big)$.

CFR determines an iteration's strategy by applying any of several \emph{regret minimization} algorithms to each infoset~\cite{Littlestone94:Weighted,Chaudhuri09:Parameter-free}.  Typically, {\em regret matching} (RM) is used as the regret minimization algorithm within CFR due to RM's simplicity and lack of parameters~\cite{Hart00:Simple}.

In RM, a player picks a distribution over actions in an infoset in proportion to the positive regret on those actions. Formally, on each iteration $t+1$, $p$ selects actions $a \in A(I)$ according to probabilities
\begin{equation}
\sigma^{t+1}(I,a) =
\frac{R^t_+(I,a)}{\sum_{a' \in A(I)}R_+^t(I,a')}
\label{eq:rm}
\end{equation}
If $\sum_{a' \in A(I)}R_+^t(I,a') = 0$ then any arbitrary strategy may be chosen. Typically each action is assigned equal probability, but in this paper we choose the action with highest counterfactual regret with probability $1$, which we find empirically helps RM better cope with approximation error (see Figure \ref{fig:ablations}).

If a player plays according to regret matching in infoset $I$ on every iteration, then on iteration $T$, $R^T(I) \le \Delta\sqrt{|A(I)|}\sqrt{T}$~\cite{Cesa-Bianchi06:Prediction}.
\citet{Zinkevich07:Regret} show that the sum of the counterfactual regret across all infosets upper bounds the total regret. Therefore,
if player~$p$ plays according to CFR on every iteration, then
$R_p^T \le \sum_{I \in \mathcal{I}_p} R^T(I)$.
So, as $T \rightarrow \infty$, $\frac{R_p^T}{T} \rightarrow 0$.

The average strategy $\bar{\sigma}_p^T(I)$ for an infoset $I$ on iteration $T$ is
$\bar{\sigma}_p^T(I) = \frac{\sum_{t = 1}^T \big(\pi_p^{\sigma^t}(I)\sigma_p^t(I)\big)}{\sum_{t = 1}^T \pi_p^{\sigma^t}(I)}$.

In two-player zero-sum games, if both players' average total regret satisfies $\frac{R_p^T}{T} \le \epsilon$, then their average strategies $\langle \bar{\sigma}^T_1, \bar{\sigma}^T_2 \rangle$ form a $2\epsilon$-Nash equilibrium~\cite{Waugh09:Thesis}. Thus, CFR constitutes an anytime algorithm for finding an $\epsilon$-Nash equilibrium in two-player zero-sum games. 

In practice, faster convergence is achieved by alternating which player updates their regrets on each iteration rather than updating the regrets of both players simultaneously each iteration, though this complicates the theory~\cite{Farina18:Online,Burch18:Revisiting}. We use the alternating-updates form of CFR in this paper.

\subsection{Monte Carlo Counterfactual Regret Minimization}
\vspace{-0.03in}
Vanilla CFR requires full traversals of the game tree, which is infeasible in large games. One method to combat this is Monte Carlo CFR (MCCFR), in which only a portion of the game tree is traversed on each iteration~\cite{Lanctot09:Monte}.  In MCCFR, a subset of nodes $Q^t$ in the game tree is traversed at each iteration, where $Q^t$ is sampled from some distribution $\mathcal{Q}$. 
Sampled regrets $\tilde{r}^t$ are tracked rather than exact regrets. For infosets that are sampled at iteration $t$, $\tilde{r}^t(I, a)$ is equal to $r^t(I,a)$ divided by the probability of having sampled $I$; for unsampled infosets $\tilde{r}^t(I,a)=0$.
See Appendix \ref{app:proof} for more details.

There exist a number of MCCFR variants~\cite{Gibson12:Generalized,Johanson12:Efficient,Jackson17:Targeted}, but for this paper we focus specifically on the \emph{external sampling} variant due to its simplicity and strong performance. In external-sampling MCCFR the game tree is traversed for one player at a time, alternating back and forth. We refer to the player who is traversing the game tree on the iteration as the \emph{traverser}. Regrets are updated only for the traverser on an iteration. At infosets where the traverser acts, all actions are explored. At other infosets and chance nodes, only a single action is explored. 

External-sampling MCCFR probabilistically converges to an equilibrium. For any $\rho \in (0,1]$, total regret is bounded by $R_p^T \le \big(1 + \frac{\sqrt{2}}{\sqrt{\rho}}\big) |\mathcal{I}_p| \Delta \sqrt{|A|}\sqrt{T}$ with probability $1-\rho$.


\section{Related Work}
\vspace{-0.03in}
CFR is not the only iterative algorithm capable of solving large imperfect-information games. First-order methods converge to a Nash equilibrium in $O(1/T)$~\cite{Hoda10:Smoothing,Kroer18:Faster,Kroer18:Solving}, which is far better than CFR's theoretical bound. However, in practice the fastest variants of CFR are substantially faster than the best first-order methods. Moreover, CFR is more robust to error and therefore likely to do better when combined with function approximation. 

Neural Fictitious Self Play (NFSP)~\cite{Heinrich16:Deep} previously combined deep learning function approximation with Fictitious Play~\cite{Brown51:Iterative} to produce an AI for heads-up limit Texas hold'em, a large imperfect-information game. However, Fictitious Play has weaker theoretical convergence guarantees than CFR, and in practice converges slower. We compare our algorithm to NFSP in this paper. Model-free policy gradient algorithms have been shown to minimize regret when parameters are tuned appropriately~\cite{Srinivasan18:Actor} and achieve performance comparable to NFSP.

Past work has investigated using deep learning to estimate values at the depth limit of a subgame in imperfect-information games~\cite{Moravcik17:DeepStack,Brown18:Depth}. However, tabular CFR was used within the subgames themselves. Large-scale function approximated CFR has also been developed for single-agent settings~\cite{Jin17:Regret}. Our algorithm is intended for the multi-agent setting and is very different from the one proposed for the single-agent setting.

Prior work has combined regression tree function approximation with CFR~\cite{Waugh15:Solving} in an algorithm called \emph{Regression CFR (RCFR)}. This algorithm defines a number of features of the infosets in a game and calculates weights to approximate the regrets that a tabular CFR implementation would produce. Regression CFR is algorithmically similar to Deep CFR, but uses hand-crafted features similar to those used in abstraction, rather than learning the features. RCFR also uses full traversals of the game tree (which is infeasible in large games) and has only been evaluated on toy games. It is therefore best viewed as the first proof of concept that function approximation can be applied to CFR.

Concurrent work has also investigated a similar combination of deep learning with CFR, in an algorithm referred to as Double Neural CFR~\cite{Li18:Double}.  However, that approach may not be theoretically sound and the authors consider only small games. There are important differences between our approaches in how training data is collected and how the behavior of CFR is approximated.

\section{Description of the Deep Counterfactual Regret Minimization Algorithm}
\vspace{-0.03in}
In this section we describe Deep CFR. The goal of Deep CFR is to approximate the behavior of CFR without calculating and accumulating regrets at each infoset, by generalizing across similar infosets using function approximation via deep neural networks.


On each iteration $t$, Deep CFR conducts a constant number $K$ of partial traversals of the game tree, with the path of the traversal determined according to external sampling MCCFR. At each infoset $I$ it encounters, it plays a strategy $\sigma^t(I)$ determined by regret matching on the output of a neural network $V: I \to \mathbf{R}^{|A|}$ defined by parameters $\theta^{t-1}_p$ that takes as input the infoset $I$ and outputs values $V(I,a|\theta^{t-1})$. Our goal is for $V(I,a|\theta^{t-1})$ to be approximately proportional to the regret $R^{t-1}(I,a)$ that tabular CFR would have produced.

When a terminal node is reached, the value is passed back up. In chance and opponent infosets, the value of the sampled action is passed back up unaltered. In traverser infosets, the value passed back up is the weighted average of all action values, where action $a$'s weight is $\sigma^t(I,a)$. This produces samples of this iteration's instantaneous regrets for various actions. Samples are added to a memory $\mathcal{M}_{v,p}$, where $p$ is the traverser, using reservoir sampling \citep{vitter1985random} if capacity is exceeded.

Consider a nice property of the sampled instantaneous regrets induced by external sampling:
\begin{lemma}
\label{lemma:value_lemma}
For external sampling MCCFR, the sampled instantaneous regrets are an unbiased estimator of the \textbf{advantage}, i.e. the difference in expected payoff for playing $a$ vs $\sigma_p^t(I)$ at $I$, assuming both players play $\sigma^t$ everywhere else. $$\mathbb{E}_{Q\in \mathcal{Q}_t} \left[ \tilde{r}_p^{\sigma^t}(I,a) \middle| Z_I \cap Q \neq \emptyset \right] = \frac{v^{\sigma^t}(I,a) - v^{\sigma^t}(I)}{\pi_{-p}^{\sigma^t}(I)}.$$
\end{lemma}

The proof is provided in Appendix \ref{app:value_lemma}. 

Recent work in deep reinforcement learning has shown that neural networks can effectively predict and generalize advantages in challenging environments with large state spaces, and use that to learn good policies \cite{mnih2016asynchronous}.

Once a player's $K$ traversals are completed, a new network is trained \emph{from scratch} to determine parameters $\theta_p^t$ by minimizing MSE between predicted advantage $V_p(I,a|\theta^t)$ and samples of instantaneous regrets from prior iterations $t' \le t$ $\tilde{r}^{t'}(I,a)$ drawn from the memory.
The average over all sampled instantaneous advantages $\tilde{r}^{t'}(I,a)$ is proportional to the total sampled regret $\tilde{R}^t(I,a)$ (across actions in an infoset), so once a sample is added to the memory it is never removed except through reservoir sampling, even when the next CFR iteration begins. 

One can use any loss function for the value and average strategy model that satisfies Bregman divergence~\cite{banerjee2005clustering}, such as mean squared error loss.

While almost any sampling scheme is acceptable so long as the samples are weighed properly, external sampling has the convenient property that it achieves both of our desired goals by assigning all samples in an iteration equal weight. Additionally, exploring all of a traverser's actions helps reduce variance. However, external sampling may be impractical in games with extremely large branching factors, so a different sampling scheme, such as outcome sampling~\cite{Lanctot09:Monte}, may be desired in those cases.

In addition to the value network, a separate policy network $\Pi: I \to \mathbf{R}^{|A|}$ approximates the average strategy at the end of the run, because it is the \emph{average strategy played over all iterations} that converges to a Nash equilibrium. To do this, we maintain a separate memory $\mathcal{M}_{\Pi}$ of sampled infoset probability vectors for both players. Whenever an infoset $I$ belonging to player $p$ is traversed during the opposing player's traversal of the game tree via external sampling, the infoset probability vector $\sigma^t(I)$ is added to $\mathcal{M}_{\Pi}$ and assigned weight $t$.

If the number of Deep CFR iterations and the size of each value network model is small, then one can avoid training the final policy network by instead storing each iteration's value network~\cite{Steinberger19:Single}. During actual play, a value network is sampled randomly and the player plays the CFR strategy resulting from the predicted advantages of that network. This eliminates the function approximation error of the final average policy network, but requires storing all prior value networks. Nevertheless, strong performance and low exploitability may still be achieved by storing only a subset of the prior value networks~\cite{Jackson16:Compact}.

Theorem~\ref{th:deepcfr_approx} states that if the memory buffer is sufficiently large, then with high probability Deep CFR will result in average regret being bounded by a constant proportional to the square root of the function approximation error.

\newcommand{\norm}[1]{\left\lVert#1\right\rVert}

\begin{theorem}
\label{th:deepcfr_approx}

Let $T$ denote the number of Deep CFR iterations, $|A|$ the maximum number of actions at any infoset, and $K$ the number of traversals per iteration. Let $\mathcal{L}^t_V$ be the average MSE loss for $V_p(I,a|\theta^t)$ on a sample in $\mathcal{M}_{V,p}$ at iteration $t$ , and let $\mathcal{L}^t_{V^*}$ be the minimum loss achievable for any function $V$. Let $\mathcal{L}^t_V - \mathcal{L}^t_{V^*} \leq \epsilon_\mathcal{L}$.

If the value memories are sufficiently large, then with probability $1-\rho $ total regret at time $T$ is bounded by

\begin{equation}
    R_p^T \leq \left( 1 + \frac{\sqrt{2}}{\sqrt{\rho K}} \right) \Delta |\mathcal{I}_p| \sqrt{|A|}\sqrt{T} +  4T |\mathcal{I}_p| \sqrt{|A| \Delta  \epsilon_\mathcal{L}}
\end{equation}
with probability $1-\rho$.

\end{theorem}

\begin{corollary}
\label{cor:deepcfr_approx}
As $T\to \infty$, average regret $\frac{R^T_p}{T}$ is bounded by $$ 4 |\mathcal{I}_p| \sqrt{|A| \Delta \epsilon_\mathcal{L}}$$ with high probability.
\end{corollary}

The proofs are provided in Appendix \ref{app:proof_approx}.


We do not provide a convergence bound for Deep CFR when using linear weighting, since the convergence rate of Linear CFR has not been shown in the Monte Carlo case. However, Figure~$\ref{fig:ablations}$ shows moderately faster convergence in practice.

\begin{algorithm*}[!ht]
\caption{Deep Counterfactual Regret Minimization}
\begin{algorithmic}
\Function{DeepCFR}{}
\State Initialize each player's advantage network $V(I,a|\theta_p)$ with parameters $\theta_p$ so that it returns 0 for all inputs.
\State Initialize reservoir-sampled advantage memories $\mathcal{M}_{V,1}, \mathcal{M}_{V,2}$ and strategy memory $\mathcal{M}_{\Pi}$.
\For{CFR iteration $t=1$ to $T$}
\For{{\bf each} player $p$}
    \For{traversal $k=1$ to $K$}
        \State \Call{Traverse}{$\emptyset, p, \theta_1, \theta_2, \mathcal{M}_{V,p}, \mathcal{M}_{\Pi}$} \Comment{Collect data from a game traversal with external sampling}
    \EndFor
  \State Train $\theta_p$ from scratch on loss $\mathcal{L}(\theta_p)=\mathbb{E}_{(I, t', \tilde{r}^{t'}) \sim \mathcal{M}_{V,p}}\left[ t' \sum_a \left( \tilde{r}^{t'}(a) - V(I, a|\theta_p)\right)^2\right]$
\EndFor
\EndFor
  \State Train $\theta_\Pi$ on loss $\mathcal{L}(\theta_\Pi)=\mathbb{E}_{(I,t',\sigma^{t'}) \sim \mathcal{M}_{\Pi}}\left[ t' \sum_a \left(\sigma^{t'}(a) - \Pi(I, a|\theta_\Pi)\right)^2\right]$ 
\State \Return $\theta_\Pi$
\EndFunction
\end{algorithmic}
\label{alg:deepcfr}
\end{algorithm*}

\begin{algorithm*}[!ht]
\caption{CFR Traversal with External Sampling}
\begin{algorithmic}
\Function{Traverse}{$h, p, \theta_1, \theta_2, \mathcal{M}_{V}, \mathcal{M}_{\Pi}$, t}
\State {\it Input:} History $h$, traverser player $p$, regret network parameters $\theta$ for each player, advantage memory $\mathcal{M}_{V}$ for player $p$, strategy memory $\mathcal{M}_{\Pi}$, CFR iteration $t$.
\\
\If{$h$ is terminal}
\State \Return the payoff to player $p$
\ElsIf{$h$ is a chance node}
    \State $a \sim \sigma(h)$
    \State \Return \Call{Traverse}{$h \cdot a, p, \theta_1, \theta_2, \mathcal{M}_{V}, \mathcal{M}_{\Pi}$, t}
\ElsIf{$P(h) = p$} \Comment{If it's the traverser's turn to act}
    \State Compute strategy $\sigma^t(I)$ from predicted advantages $V(I(h),a|\theta_p)$ using regret matching. 
    \For{$a \in A(h)$}
        \State $v(a) \gets$ \Call{Traverse}{$h \cdot a,p,\theta_1,\theta_2,\mathcal{M}_{V},\mathcal{M}_{\Pi}$, t} \Comment{Traverse each action}
	\EndFor
	\For{$a \in A(h)$}
	    \State $\tilde{r}(I,a) \gets v(a) - \sum_{a'\in A(h)}{\sigma(I,a')\cdot v(a')}$ \Comment{Compute advantages}
    \EndFor
    \State Insert the infoset and its action advantages $(I, t, \tilde{r}^t(I))$ into the advantage memory $\mathcal{M}_{V}$
\Else \Comment{If it's the opponent's turn to act}
    \State Compute strategy $\sigma^t(I)$ from predicted advantages $V(I(h), a|\theta_{3-p})$ using regret matching.
    \State Insert the infoset and its action probabilities $(I, t, \sigma^t(I))$ into the strategy memory $\mathcal{M}_{\Pi}$
    \State Sample an action $a$ from the probability distribution $\sigma^t(I)$.
    \State \Return \Call{Traverse}{$h \cdot a, p, \theta_1, \theta_2, \mathcal{M}_{V}, \mathcal{M}_{\Pi}$, t}
\EndIf
\EndFunction

\end{algorithmic}
\end{algorithm*}

\section{Experimental Setup}
\vspace{-0.03in}
We measure the performance of Deep CFR (Algorithm $\ref{alg:deepcfr}$) in approximating an equilibrium in heads-up flop hold'em poker (FHP). FHP is a large game with over $10^{12}$ nodes and over $10^9$ infosets. In contrast, the network we use has 98,948 parameters. FHP is similar to heads-up limit Texas hold'em (HULH) poker, but ends after the second betting round rather than the fourth, with only three community cards ever dealt. We also measure performance relative to domain-specific abstraction techniques in the benchmark domain of HULH poker, which has over $10^{17}$ nodes and over $10^{14}$ infosets. The rules for FHP and HULH are given in Appendix~\ref{sec:rules}.

In both games, we compare performance to NFSP, which is the previous leading algorithm for imperfect-information game solving using domain-independent function approximation, as well as state-of-the-art abstraction techniques designed for the domain of poker~\cite{Johanson13:Evaluating,Ganzfried14:Potential-Aware,Brown15:Hierarchical}.
\subsection{Network Architecture}
\vspace{-0.03in}
We use the neural network architecture shown in Figure \ref{fig:nn_arch} for both the value network $V$ that computes advantages for each player and the network $\Pi$ that approximates the final average strategy. This network has a depth of 7 layers and 98,948 parameters. Infosets consist of sets of cards and bet history. The cards are represented as the sum of three embeddings: a rank embedding (1-13), a suit embedding (1-4), and a card embedding (1-52). These embeddings are summed for each set of permutation invariant cards (hole, flop, turn, river), and these are concatenated. In each of the $N_\textit{rounds}$ rounds of betting there can be at most $6$ sequential actions, leading to $6 N_\textit{rounds}$ total unique betting positions. Each betting position is encoded by a binary value specifying whether a bet has occurred, and a float value specifying the bet size.

The neural network model begins with separate branches for the cards and bets, with three and two layers respectively. Features from the two branches are combined and three additional fully connected layers are applied. Each fully-connected layer consists of $x_{i+1} = \textrm{ReLU}( A x [ + x ] )$. The optional skip connection $[+ x]$ is applied only on layers that have equal input and output dimension. Normalization (to zero mean and unit variance) is applied to the last-layer features. The network architecture was not highly tuned, but normalization and skip connections were used because they were found to be important to encourage fast convergence when running preliminary experiments on pre-computed equilibrium strategies in FHP. A full network specification is provided in Appendix \ref{app:network}.

In the value network, the vector of outputs represented predicted advantages for each action at the input infoset. In the average strategy network, outputs are interpreted as logits of the probability distribution over actions. 

\begin{figure}[t]
\centering
\includegraphics[width=\columnwidth]{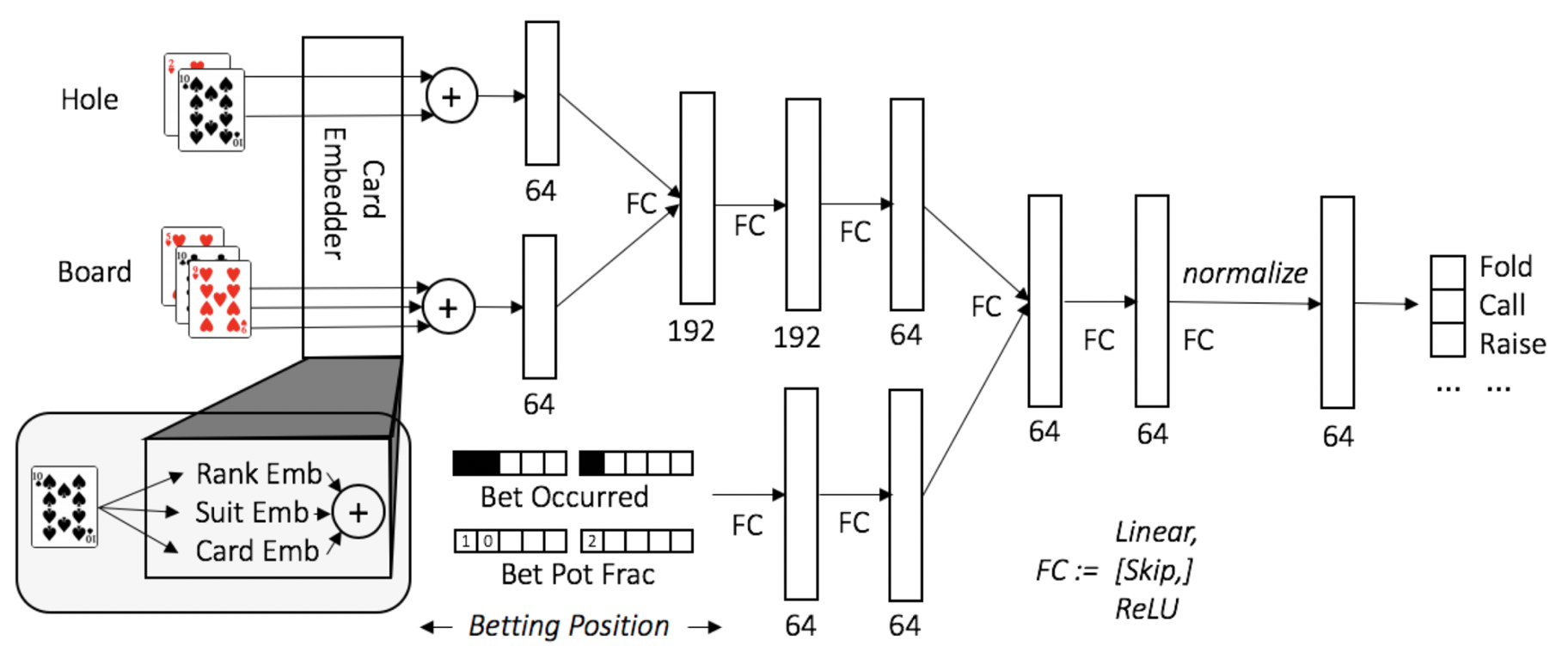}
\label{fig:nn_arch}
\vspace{-3mm}
\caption{\small{The neural network architecture used for Deep CFR. The network takes an infoset (observed cards and bet history) as input and outputs values (advantages or probability logits) for each possible action.}}
\end{figure}

\subsection{Model training}
\vspace{-0.03in}
We allocate a maximum size of 40 million infosets to each player's advantage memory $\mathcal{M}_{V,p}$  and the strategy memory $\mathcal{M}_{\Pi}$.
The value model is trained from scratch each CFR iteration, starting from a random initialization. We perform 4,000 mini-batch stochastic gradient descent (SGD) iterations using a batch size of 10,000 and perform parameter updates using the Adam optimizer \cite{Kingma14:Adam} with a learning rate of $0.001$, with gradient norm clipping to $1$. For HULH we use 32,000 SGD iterations and a batch size of 20,000. Figure~\ref{fig:ablations} shows that training the model from scratch at each iteration, rather than using the weights from the previous iteration, leads to better convergence.

\subsection{Linear CFR}
\vspace{-0.03in}
There exist a number of variants of CFR that achieve much faster performance than vanilla CFR. However, most of these faster variants of CFR do not handle approximation error well~\cite{Tammelin15:Solving,Burch17:Time,Brown19:Solving,Schmid19:Variance}. In this paper we use \emph{Linear CFR (LCFR)}~\cite{Brown19:Solving}, a variant of CFR that is faster than CFR and in certain settings is the fastest-known variant of CFR (particularly in settings with wide distributions in payoffs), and which tolerates approximation error well. LCFR is not essential and does not appear to lead to better performance asymptotically, but does result in faster convergence in our experiments.

LCFR is like CFR except iteration $t$ is weighed by $t$. Specifically, we maintain a \emph{weight} on each entry stored in the advantage memory and the strategy memory, equal to $t$ when this entry was added. When training $\theta_p$ each iteration $T$, we rescale all the batch weights by $\frac{2}{T}$ and minimize weighted error.


\section{Experimental Results}
\vspace{-0.03in}
Figure~\ref{fig:abstraction} compares the performance of Deep CFR to different-sized domain-specific abstractions in FHP. The abstractions are solved using external-sampling Linear Monte Carlo CFR~\cite{Lanctot09:Monte,Brown19:Solving}, which is the leading algorithm in this setting. The 40,000 cluster abstraction means that the more than $10^9$ different decisions in the game were clustered into 40,000 abstract decisions, where situations in the same bucket are treated identically. This bucketing is done using K-means clustering on domain-specific features. The \emph{lossless abstraction} only clusters together situations that are strategically isomorphic (e.g., flushes that differ only by suit), so a solution to this abstraction maps to a solution in the full game without error.

Performance and exploitability are measured in terms of milli big blinds per game (mbb/g), which is a standard measure of win rate in poker.

The figure shows that Deep CFR asymptotically reaches a similar level of exploitability as the abstraction that uses 3.6 million clusters, but converges substantially faster. Although Deep CFR is more efficient in terms of nodes touched, neural network inference and training requires considerable overhead that tabular CFR avoids. However, Deep CFR does not require advanced domain knowledge. We show Deep CFR performance for 10,000 CFR traversals per step. Using more traversals per step is less sample efficient and requires greater neural network training time but requires fewer CFR steps.

Figure~\ref{fig:abstraction} also compares the performance of Deep CFR to NFSP, an existing method for learning approximate Nash equilibria in imperfect-information games. NFSP approximates fictitious self-play, which is proven to converge to a Nash equilibrium but in practice does so far slower than CFR. We observe that Deep CFR reaches an exploitability of 37 mbb/g while NFSP converges to 47 mbb/g.\footnote{We run NFSP with the same model architecture as we use for Deep CFR. In the benchmark game of Leduc Hold'em, our implementation of NFSP achieves an average exploitability (total exploitability divided by two) of 37 mbb/g in the benchmark game of Leduc Hold'em, which is substantially lower than originally reported in \citet{Heinrich16:Deep}. We report NFSP's best performance in FHP across a sweep of hyperparameters.} We also observe that Deep CFR is more sample efficient than NFSP. However, these methods spend most of their wallclock time performing SGD steps, so in our implementation we see a less dramatic improvement over NFSP in wallclock time than sample efficiency. 

Figure~\ref{fig:traversals} shows the performance of Deep CFR using different numbers of game traversals, network SGD steps, and model size. As the number of CFR traversals per iteration is reduced, convergence becomes slower but the model converges to the same final exploitability. This is presumably because it takes more iterations to collect enough data to reduce the variance sufficiently. On the other hand, reducing the number of SGD steps does not change the rate of convergence but affects the asymptotic exploitability of the model. This is presumably because the model loss decreases as the number of training steps is increased per iteration (see Theorem \ref{th:deepcfr_approx}). Increasing the model size also decreases final exploitability up to a certain model size in FHP. 

In Figure~\ref{fig:ablations} we consider ablations of certain components of Deep CFR. Retraining the regret model from scratch at each CFR iteration converges to a substantially lower exploitability than fine-tuning a single model across all iterations. We suspect that this is because a single model gets stuck in bad local minima as the objective is changed from iteration to iteration. The choice of reservoir sampling to update the memories is shown to be crucial; if a sliding window memory is used, the exploitability begins to increase once the memory is filled up, even if the memory is large enough to hold the samples from many CFR iterations.

Finally, we measure head-to-head performance in HULH. We compare Deep CFR and NFSP to the approximate solutions (solved via Linear Monte Carlo CFR) of three different-sized abstractions: one in which the more than $10^{14}$ decisions are clustered into $3.3 \cdot 10^6$ buckets, one in which there are $3.3 \cdot 10^7$ buckets and one in which there are $3.3 \cdot 10^8$ buckets. The results are presented in Table~\ref{tab:limit}. For comparison, the largest abstractions used by the poker AI Polaris in its 2007 HULH man-machine competition against human professionals contained roughly $3 \cdot 10^8$ buckets. When variance-reduction techniques were applied, the results showed that the professional human competitors lost to the 2007 Polaris AI by about $52 \pm 10$ mbb/g~\cite{Johanson16:Robust}. In contrast, our Deep CFR agent loses to a $3.3 \cdot 10^8$ bucket abstraction by only $-11 \pm 2$ mbb/g and beats NFSP by $43 \pm 2$ mbb/g.


\begin{table*}[t]
\small
\begin{center}
\begin{tabular}{ l | c c | c c c}
\toprule
               & \multicolumn{4}{c}{\bf Opponent Model} & \\
               &      & \multicolumn{1}{ c}{} & \multicolumn{3}{| c }{ Abstraction Size} \\
{\bf Model}    & NFSP & Deep CFR & \multicolumn{1}{| c}{ $3.3 \cdot 10^6$} & $3.3 \cdot 10^7$ & $3.3 \cdot 10^8$ \\
\midrule
NFSP           & - & $-43 \pm 2$ mbb/g &$ -40 \pm 2 $ mbb/g&$ -49 \pm 2 $ mbb/g&$ -55 \pm 2 $ mbb/g \\
Deep CFR       & $+43 \pm 2$ mbb/g & - &$ \ +6 \pm 2 $ mbb/g&$ \ -6 \pm 2 $ mbb/g&$ -11 \pm 2 $ mbb/g \\
\bottomrule
\end{tabular}
\end{center}
\vspace{-3mm}
\caption{Head-to-head expected value of NFSP and Deep CFR in HULH against converged CFR equilibria with varying abstraction sizes. For comparison, in 2007 an AI using abstractions of roughly $3 \cdot 10^8$ buckets defeated human professionals by about $52$ mbb/g (after variance reduction techniques were applied).
}
\label{tab:limit}
\end{table*}

\begin{figure}[h!]
\centering
\includegraphics[width=3.3in]{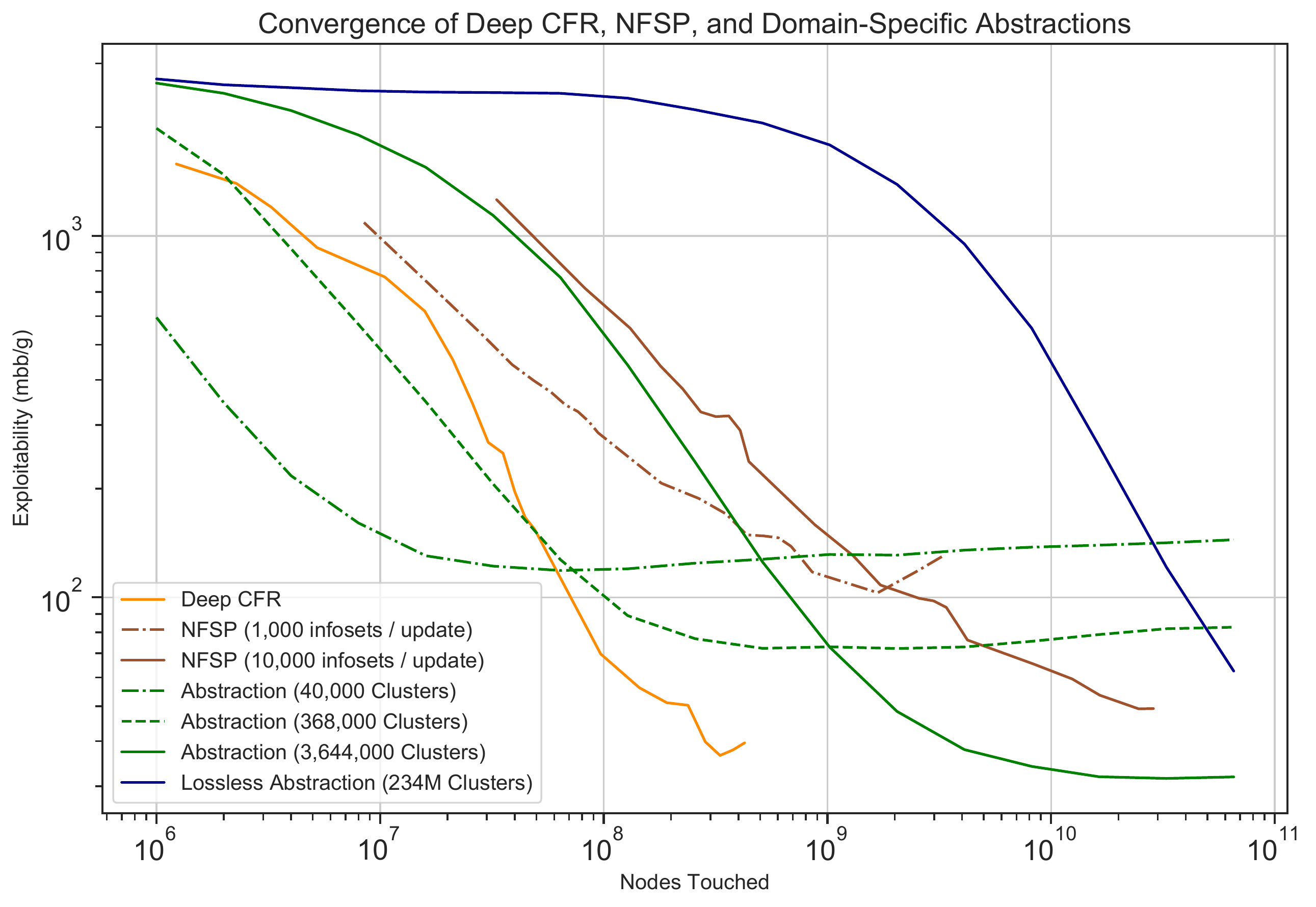}
\vspace{-3mm}
\caption{\small{Comparison of Deep CFR with domain-specific tabular abstractions and NFSP in FHP. Coarser abstractions converge faster but are more exploitable. Deep CFR converges with 2-3 orders of magnitude fewer samples than a lossless abstraction, and performs competitively with a 3.6 million cluster abstraction. Deep CFR achieves lower exploitability than NFSP, while traversing fewer infosets.}}
\label{fig:abstraction}
\end{figure}

\begin{figure*}[h]
\centering
\hspace{-.5cm}
\includegraphics[width=2.3in]{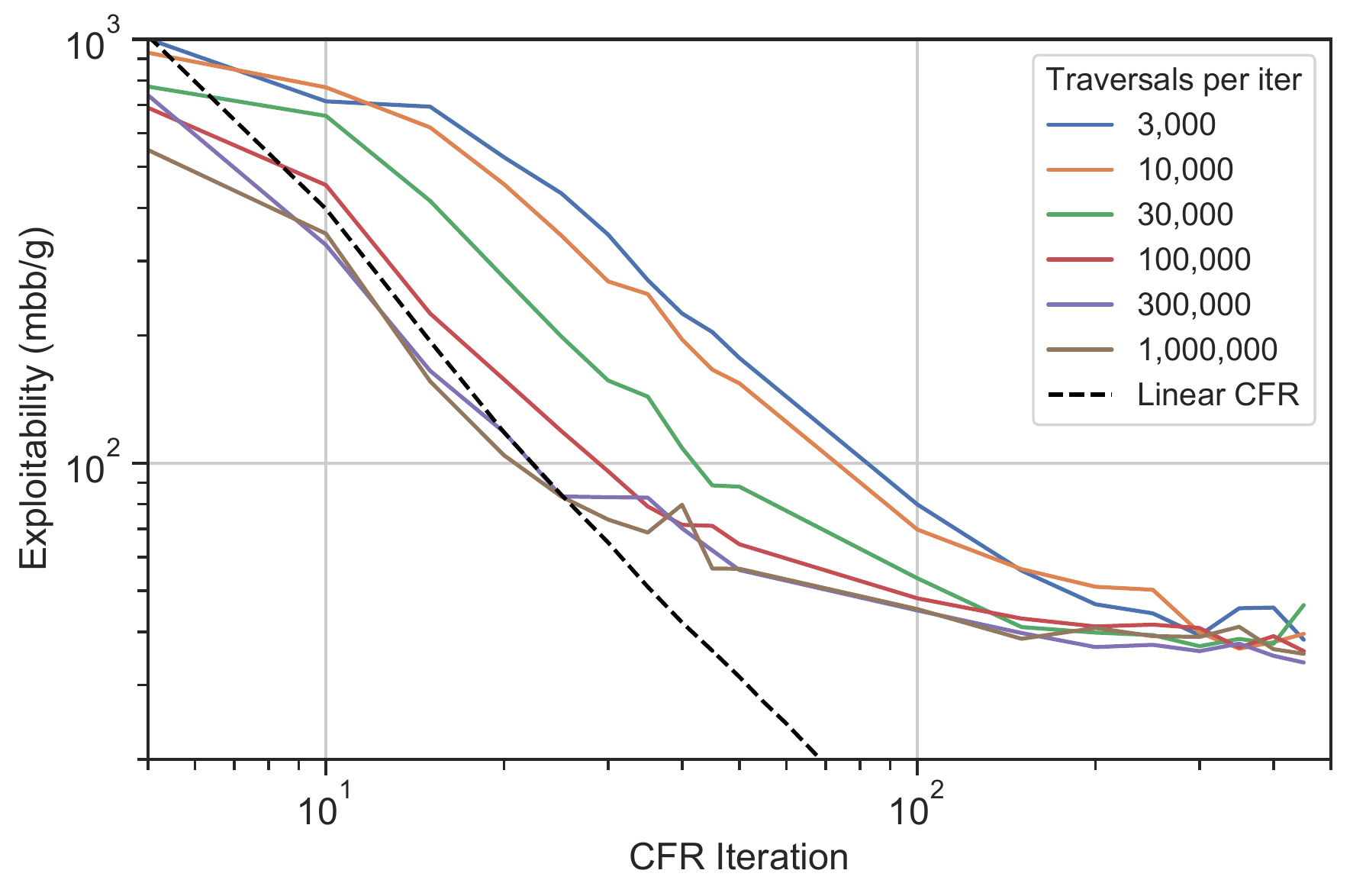}
\hspace{-.2cm}
\includegraphics[width=2.3in]{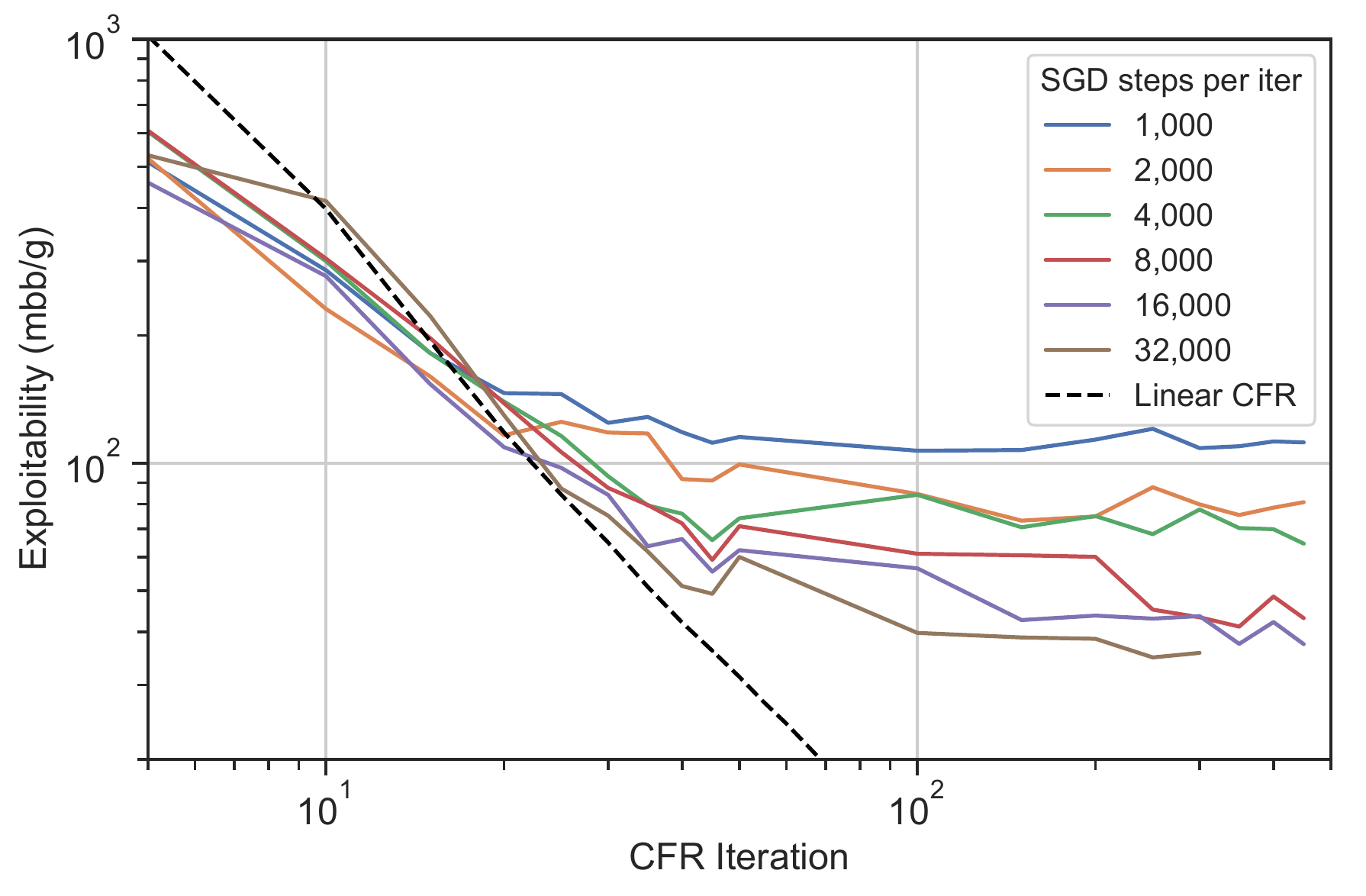}
\hspace{-.2cm}
\includegraphics[width=2.3in]{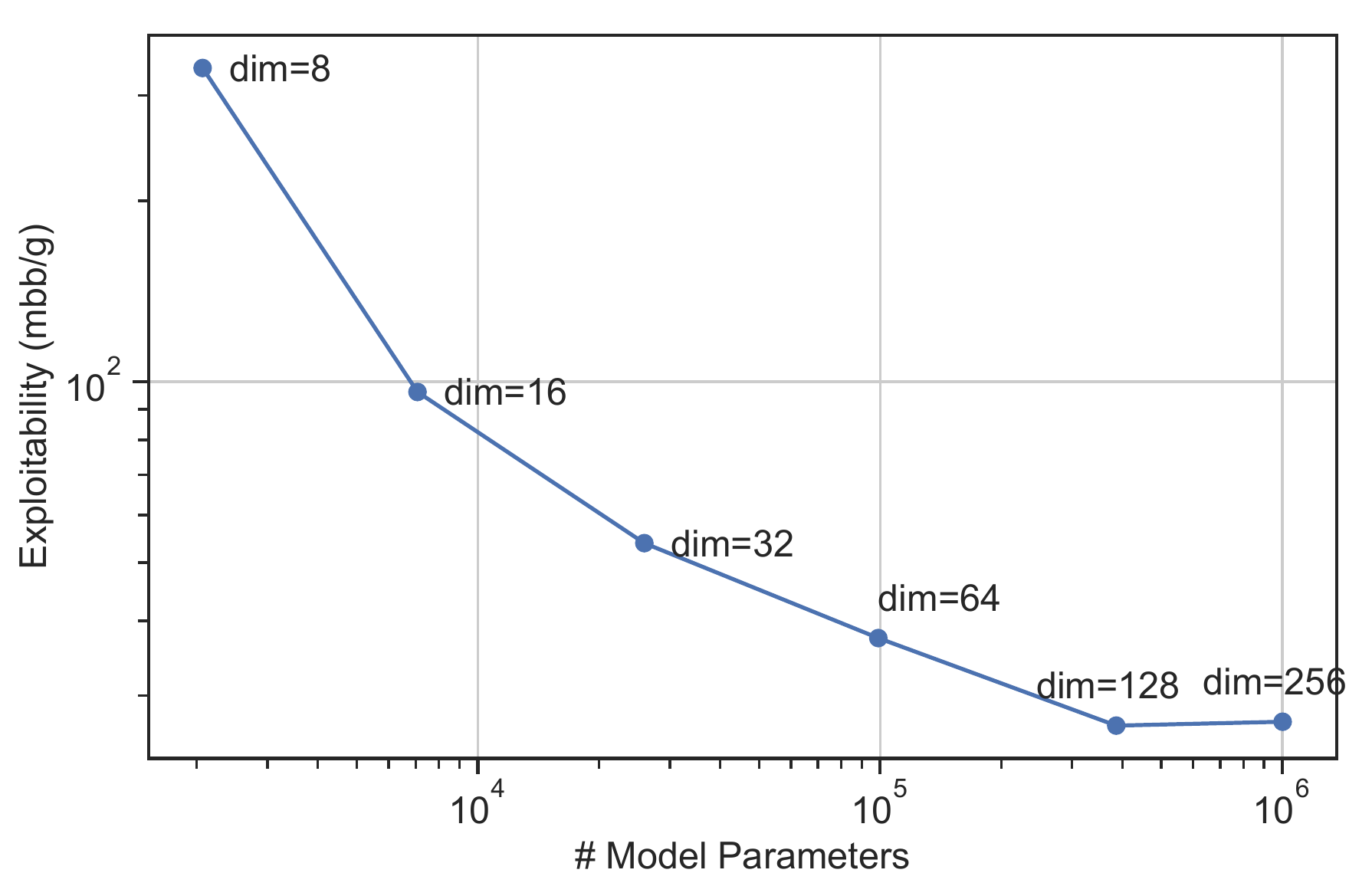}

\vspace{-3mm}
\caption{\small{\textbf{Left:} FHP convergence for different numbers of training data collection traversals per simulated LCFR iteration. The dotted line shows the performance of vanilla tabular Linear CFR without abstraction or sampling. \textbf{Middle:} FHP convergence using different numbers of minibatch SGD updates to train the advantage model at each LCFR iteration. \textbf{Right:} Exploitability of Deep CFR in FHP for different model sizes. Label indicates the dimension (number of features) in each hidden layer of the model.}}
\label{fig:traversals}
\end{figure*}


\begin{figure*}[h!]
\hspace{1cm}
\includegraphics[width=2.5in]{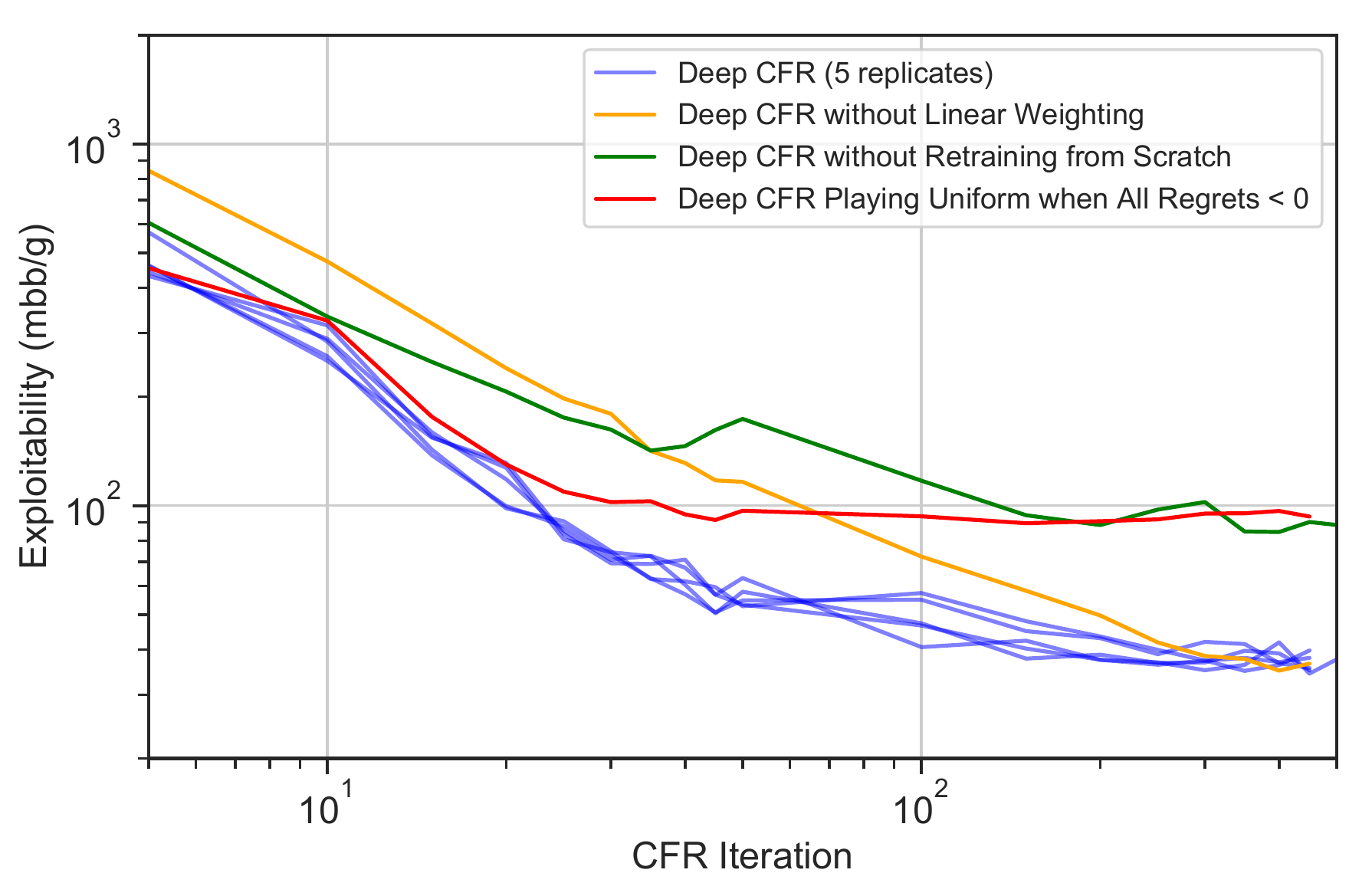}
\hspace{.5cm}
\includegraphics[width=2.5in]{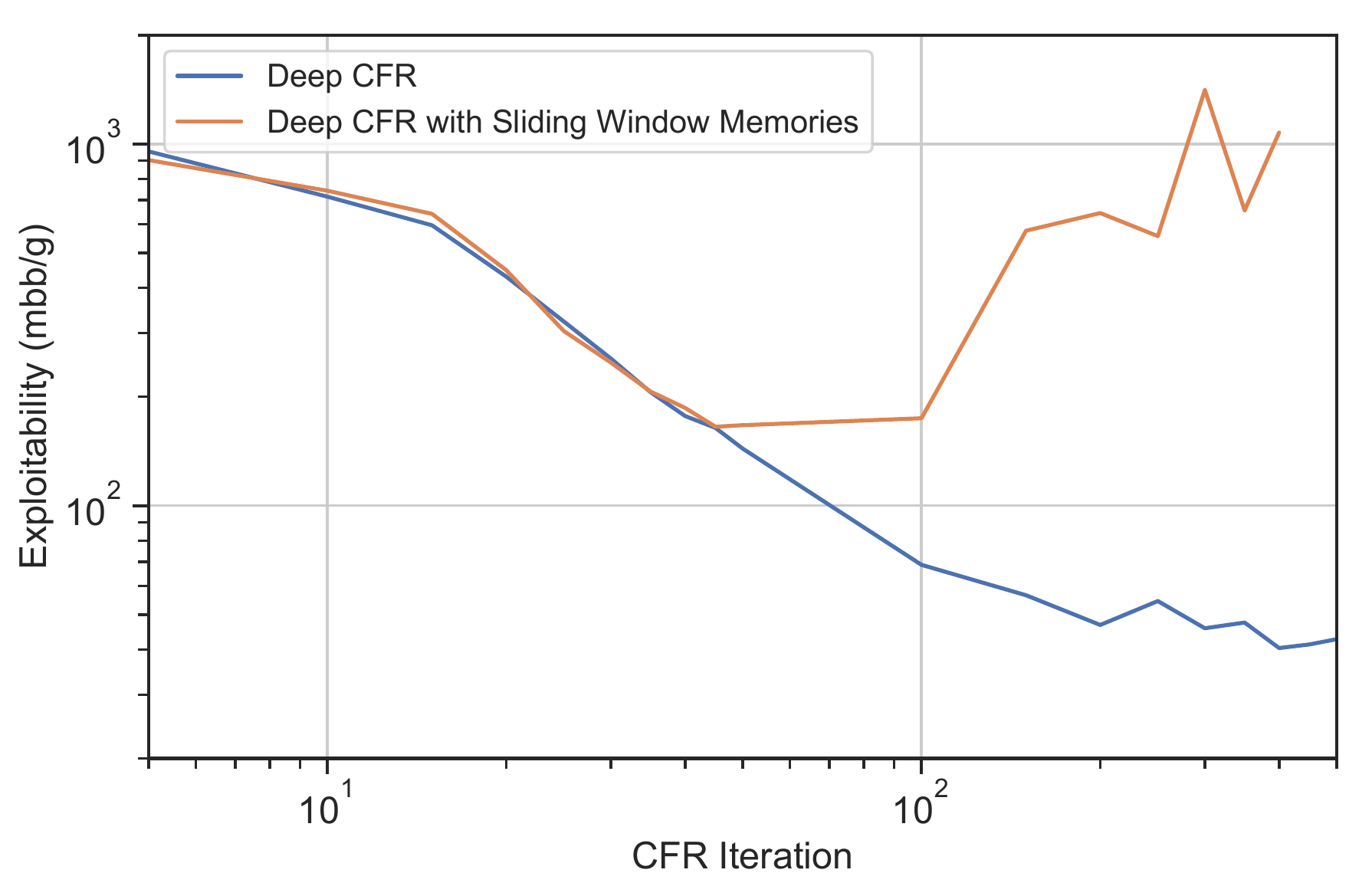}
\vspace{-3mm}
\caption{\small{Ablations of Deep CFR components in FHP. \textbf{Left:} As a baseline, we plot 5 replicates of Deep CFR, which show consistent exploitability curves (standard deviation at $t=450$ is $2.25$ mbb/g). Deep CFR without linear weighting converges to a similar exploitability, but more slowly. If the same network is fine-tuned at each CFR iteration rather than training from scratch, the final exploitability is about 50\% higher. Also, if the algorithm plays a uniform strategy when all regrets are negative (i.e. standard regret matching), rather than the highest-regret action, the final exploitability is also 50\% higher. \textbf{Right:} If Deep CFR is performed using sliding-window memories, exploitability stops converging once the buffer becomes full\footnote{With our standard hyperparameters, the memory fills up in just 1-2 iterations, so for this experiment we perform only 10,000 CFR traversals per iteration to demonstrate that convergence ceases once the buffer fills up after about 50 iterations.}. However, with reservoir sampling, convergence continues after the memories are full.}}
\label{fig:ablations}
\end{figure*}



\section{Conclusions}
\vspace{-0.03in}

We describe a method to find approximate equilibria in large imperfect-information games by combining the CFR algorithm with deep neural network function approximation. This method is theoretically principled and achieves strong performance in large poker games relative to domain-specific abstraction techniques without relying on advanced domain knowledge. This is the first non-tabular variant of CFR to be successful in large games.

Deep CFR and other neural methods for imperfect-information games provide a promising direction for tackling large games
whose state or action spaces are too large for tabular methods and where abstraction is not straightforward. Extending Deep CFR to
larger games will likely require more scalable sampling strategies than those used in this work, as well as strategies to reduce the high variance in sampled payoffs. Recent work has suggested promising directions both for more scalable sampling \cite{Li18:Double} and variance reduction techniques \cite{Schmid19:Variance}. We believe these are important areas for future work.

\clearpage

\bibliography{dairefs}
\bibliographystyle{icml2019}

\clearpage

\newpage

\appendix

\onecolumn

\section{Rules for Heads-Up Limit Texas Hold'em and Flop Hold'em Poker}
\label{sec:rules}
Heads-up limit Texas hold'em is a two-player zero-sum game. There are two players and the position of the two players alternate after each hand. On each betting round, each player can choose to either fold, call, or raise. Folding results in the player losing and the money in the pot being awarded to the other player. Calling means the player places a number of chips in the pot equal to the opponent's share. Raising means that player adds more chips to the pot than the opponent's share. A round ends when a player calls (if both players have acted). There cannot be more than three raises in the first or second betting round or more than four raises in the third or fourth betting round, so there is a limited number of actions in the game. Raises in the first two rounds are \$100 and raises in the second two rounds are \$200.

At the start of each hand of HULH, both players are dealt two private cards from a standard 52-card deck. $P_1$ must place \$50 in the pot and $P_2$ must place \$100 in the pot. A round of betting then occurs starting with $P_1$. When the round ends, three \emph{community} cards are dealt face up that both players can ultimately use in their final hands. Another round of betting occurs, starting with $P_2$ this time. Afterward another community card is dealt face up and another betting round occurs. Then a final card is dealt face up and a final betting round occurs. At the end of the betting round, unless a player has folded, the player with the best five-card poker hand constructed from their two private cards and the five community cards wins the pot. In the case of a tie, the pot is split evenly.

Flop Hold'em Poker is identical to HULH except there are only the first two betting rounds.

\section{Proofs of Theorems}
\label{app:proof}
\subsection{Review of MCCFR}

We begin by reviewing the derivation of convergence bounds for external sampling MCCFR from \citealt{Lanctot09:Monte}.

An MCCFR scheme is completely specified by a set of \emph{blocks} $\mathcal{Q}=\{Q_i\}$ which each comprise a subset of all terminal histories $Z$. On each iteration MCCFR samples one of these blocks, and only considers terminal histories within that block. Let $q_j > 0$ be the probability of considering block $Q_j$ in an iteration. 

Let $Z_I$ be the set of terminal nodes that contain a prefix in $I$, and let $z[I]$ be that prefix. 
Define $\pi^\sigma(h\to z)$ as the probability of playing to $z$ given that player $p$ is at node $h$ with both players playing $\sigma$. $$\pi^\sigma(h\to z) = \sum_{z\in Z_I} \frac{\pi^\sigma(z[I])}{\pi^\sigma(I)}  \pi^\sigma(z).$$ 
$\pi^\sigma(I\to z)$ is undefined when $\pi(I)=0$.

Let $q(z)=\sum_{j:z\in Q_j}{q_j}$ be the probability that terminal history $z$ is sampled in an iteration of MCCFR. For external sampling MCCFR, $q(z)=\pi^\sigma_{-i}(z)$.

The \emph{sampled value} $\tilde{v}^\sigma_i(I|j)$ when sampling block $j$ is

\begin{equation}
    \tilde{v}_p^\sigma(I|j) = \sum_{z\in Q_j \cap Z_I}{
        \frac{1}{q(z)}
        u_p(z)
        \pi^\sigma_{-p}(z[I])
        \pi^\sigma(z[I] \to z)
    }
\end{equation}

For external sampling, the sampled value reduces to
\begin{equation}
        \tilde{v}^\sigma_p(I|j) = \sum_{z\in Q_j \cap Z_I}{
        u_p(z)
        \pi_p^\sigma(z[I] \to z)
    }
\end{equation}

The sampled value is an unbiased estimator of the true value $v_p(I)$. Therefore the \textit{sampled instantaneous regret} $\tilde{r}^t(I,a) = \tilde{v}_p^{\sigma^t}(I,a) - \tilde{v}_p^{\sigma^t}(I)$ is an unbiased estimator of $r^t(I,a)$.

The \textit{sampled regret} is calculated as $\tilde{R}^T(I,a)=\sum_{t=1}^T \tilde{r}^t(I,a)$.

We first state the general bound shown in \cite{lanctot2013monte}, Theorem 3.

\citealt{lanctot2013monte} defines $\mathcal{B}_p$ to be a set with one element per distinct action sequence $\vec{a}$ played by $p$, containing all infosets that may arise when $p$ plays $\vec{a}$. $M_p$ is then defined by $\sum_{B \in \mathcal{B}_p} |B|$.   Let $\Delta$ be the difference between the maximum and minimum payoffs in the game.

\begin{theorem} (\citealt{lanctot2013monte}, Theorem 3)
For any $p \in (0,1]$, when using any algorithm in the MCCFR family such that for all $Q\in \mathcal{Q}$ and B $\in \mathcal{B}_p$,
\begin{equation}
    \sum_{I\in B}\left( 
        \sum_{z \in Q \cap Z_I}{
            \frac{
                \pi^\sigma(z[I] \to z)\pi^\sigma_{-p}(z[I])
            }{
                q(z)
            }
        }
    \right)^2 \leq \frac{1}{\delta^2}
\label{thm:lanctot_general_cond}
\end{equation}

where $\delta \leq 1$, then with probability at least $1-\rho$, total regret is bounded by

\begin{equation}
    R_p^T \leq
    \left( M_p + \frac{\sqrt{2|\mathcal{I}_p||\mathcal{B}_p|}}{\sqrt{\rho}} \right)
    \left( \frac{1}{\delta} \right)
    \Delta \sqrt{|A|T}
\label{thm:lanctot_general}
\end{equation}
.
\end{theorem}

For the case of external sampling MCCFR, $q(z)=\pi^\sigma_{-i}(z)$. \citealt{Lanctot09:Monte}, Theorem 9 shows that for external sampling, for which $q(z)=\pi^\sigma_{-i}(z)$, the inequality in (\ref{thm:lanctot_general_cond}) holds for $\delta=1$, and thus the bound implied by (\ref{thm:lanctot_general}) is

\begin{align}
    \bar{R}_p^T \leq &
    \left( M_p + \frac{\sqrt{2|\mathcal{I}_p||\mathcal{B}_p|}}{\sqrt{\rho}} \right)
    \Delta \frac{\sqrt{|A|}}{\sqrt{T}} \\
    \leq & \left(1 + \frac{\sqrt{2}}{\sqrt{\rho K}}\right)\Delta |\mathcal{I}_p| \frac{\sqrt{|A|}}{\sqrt{T}} \hspace{2cm} \textrm{because $|\mathcal{B}_p| \leq M_p \leq |\mathcal{I}_p|$}
\label{thm:lanctot_external1}
\end{align}
.

\subsection{Proof of Lemma \ref{lemma:value_lemma}}
\label{app:value_lemma}

We show $$\mathbb{E}_{Q_j\sim \mathcal{Q}} \left[ \tilde{v}_p^{\sigma^t}(I) \middle|  Z_I \cap Q_j \neq \emptyset \right] = v^{\sigma^t}(I) / \pi_{-p}^{\sigma^t}(I).$$

Let $q_j = P(Q_j)$.
\begin{align*}
    \mathbb{E}_{Q_j\sim \mathcal{Q}} \left[ \tilde{v}_p^{\sigma^t}(I) \middle| Z_I \cap Q \neq \emptyset \right] & = \frac{\mathbb{E}_{Q_j\sim \mathcal{Q}} \left[ \tilde{v}_p^{\sigma^t}(I) \right] }{ P_{Q_j \sim \mathcal{Q}}(Z_I \cap Q_j \neq \emptyset)} \\
    & = \frac{\sum_{Q_j\in \mathcal{Q}} q_j \sum_{z\in Z_I \cap Q_j} u_p(z) \pi_{-p}^{\sigma^t}(z[I]) \pi^{\sigma^t}(z[I] \to z) / q(z) }{ \pi_{-p}^{\sigma^t}(I)} \\
    & = \frac{  \sum_{z\in Z_I \cap Q_j} \left( \sum_{Q_j: z \in Q_j} q_j \right) u_p(z) \pi_{-p}^{\sigma^t}(z[I]) \pi^{\sigma^t}(z[I] \to z) / q(z) }{ \pi_{-p}^{\sigma^t}(I)} \\
    & = \frac{ \sum_{z\in Z_I} q(z) u_p(z) \pi_{-p}^{\sigma^t}(z[I]) \pi^{\sigma^t}(z[I] \to z) / q(z) }{ \pi_{-p}^{\sigma^t}(I)} \hspace{1cm} \textrm{By definition of $q(z)$} \\
    & = \frac{ v^{\sigma^t}(I) }{\pi_{-p}^{\sigma^t}(I)}
\end{align*}

The result now follows directly.

\subsection{$K$-external sampling}

\newcommand{\Rp}[1]{R_{#1}^+(a)}
\newcommand{\Rsum}[1]{\Rp{{\Sigma,#1}}}
\newcommand{\aA}{\sum_{a\in A}}

We first show that performing MCCFR with $K$ external sampling traversals per iteration ($K$-ES) shares a similar convergence bound with standard external sampling (i.e. $1$-ES). We will refer to this result in the next section when we consider the full Deep CFR algorithm. This convergence bound is rather obvious and the derivation pedantic, so the reader is welcome to skip this section.

We model $T$ rounds of $K$-external sampling as $T\times K$ rounds of external sampling, where at each round $t \cdot K + d$ (for integer $t \geq 0$ and integer $0 \leq d < K$) we play

\begin{align}
\label{eq:rm2}
\sigma_{tK+d}(a) = 
\begin{cases}
\frac{\Rp{tK}}{R_{\Sigma,{tK}}^+}\textrm{\ if\ } R_{\Sigma,{tK}}^+ > 0 \\
\textrm{\ arbitrary,\ otherwise}
\end{cases}
\end{align}

In prior work, $\sigma$ is typically defined to play $\frac{1}{|A|}$ when $\Rsum{T} \leq 0$, but in fact the convergence bounds do not constraint $\sigma$'s play in these situations, which we will demonstrate explicitly here. We need this fact because minimizing the loss $\mathcal{L}(V)$ is defined only over the samples of (visited) infosets and thus does not constrain the strategy in unvisited infosets.

\begin{lemma}
\label{lemma:rmneg}
If regret matching is used in $K$-ES, then for $0 \leq d < K$
\begin{equation}
    \sum_{a\in A} \Rp{tK} r_{tK+d}(a) \leq 0
\end{equation}
\end{lemma}

\begin{proof}
If $R_{\Sigma,{tK}}^+ \leq 0$, then $R_{tK}^+(a)=0$ for all $a$ and the result follows directly. For $R_{\Sigma,{tK}}^+ > 0$,

\begin{align}
    \aA R_{tK}^+(a) r_{tK+d}(a) & = \aA R_T^+(a)(u_{tK+d}(a)-u_{tK+d}(\sigma_{tK})) \\
    & = \left( \aA R_{tK}^+(a) u_{tK+d}(a) \right) - 
         \left( u_{tK+d}(\sigma_{tK})\aA R_{tK}^+(a) \right) \\
    & = \left( \aA R_{tK}^+(a) u_{tK+d}(a) \right) - 
         \left( \aA \sigma_{tK+d}(a)u_{tK+d}(a) \right) \Rsum{tK} \\
    & = \left( \aA R_{tK}^+(a) u_{tK+d}(a) \right) - 
         \left( \aA \frac{\Rp{tK}}{\Rsum{tK}} u_{tK+d}(a) \right) \Rsum{tK} \\
    & = \left( \aA R_{tK}^+(a) u_{tK+d}(a) \right) - 
         \left( \aA \Rp{tK}(a) u_{tK+d}(a) \right) \\
    & = 0
\end{align}
\end{proof} 

\begin{theorem}
\label{th:blackwell}
Playing according to Equation \ref{eq:rm2} guarantees the following bound on total regret
\begin{equation}
    \sum_{a\in A}(\Rp{TK})^2 \leq |A|\Delta^2 K^2 T
\end{equation}
\end{theorem}

\begin{proof}

We prove by recursion on $T$.
\begin{align}
    \sum_{a\in A}(\Rp{TK})^2 \leq & 
    \sum_{a\in A}\left(\Rp{(T-1)K} + \sum_{d=0}^{K-1}{r_{tK-d}(a)} \right)^2 \\
    = & \sum_{a\in A} \biggl( R_{(T-1)K}^+(a)^2 +
        2 \sum_{d=0}^{K-1}r_d(a) R_{(T-1)K}^+(a) + \sum_{d=0}^{K-1}\sum_{d'=0}^{K-1}r_{TK-d}(a) r_{TK-d'}(a) \biggr)
\end{align}

By Lemma \ref{lemma:rmneg}, 

\begin{equation}
    \sum_{a\in A}(R_{TK}^+(a))^2 \leq 
            \sum_{a\in A} (R_{(T-1)K}^+(a))^2 +
            \sum_{a\in A} \sum_{d=0}^{K-1}\sum_{d'=0}^{K-1}r_{TK-d}(a) r_{TK-d'}(a)
\end{equation}

By induction,
\begin{equation}
    \sum_{a\in A}(R_{(T-1)K}^+(a))^2 \leq |A|\Delta^2 (T-1)
\end{equation}

From the definition, $|r_{TK-d}(a)|\leq \Delta$

\begin{align}
    \sum_{a\in A}(R_{TK}^+(a))^2 \leq 
           |A|\Delta^2 (T-1) + K^2 |A| \Delta^2
            = |A| \Delta^2 K^2 T 
\end{align}

\end{proof}



\begin{theorem} (\citealt{lanctot2013monte}, Theorem 3 \& Theorem 5)
\label{thm:mccfr_bound}
After $T$ iterations of $K$-ES, average regret is bounded by

\begin{equation}
    \bar{R}_p^{TK} \leq
    \left( 1 + \frac{\sqrt{2}}{\sqrt{\rho K}} \right)
    |\mathcal{I}_p|\Delta \frac{\sqrt{|A|}}{\sqrt{T}}
\label{thm:lanctot_external2}
\end{equation}

with probability $1-\rho$.

\end{theorem}

\begin{proof}
The proof follows \citealt{lanctot2013monte}, Theorem 3. Note that $K$-ES is only different from ES in terms of the choice of $\sigma_T$, and the proof in \citealt{lanctot2013monte} only makes use of $\sigma_T$ via the bound on $(\sum_a R_+^T(a) )^2$ that we showed in Theorem \ref{th:blackwell}. Therefore, we can apply the same reasoning to arrive at

\begin{equation}
\tilde{R}_p^{TK} \leq
\frac{\Delta M_p \sqrt{|A|T} K}{\delta}
\end{equation}
(\citealt{lanctot2013monte}, Eq. (4.30)).

\citealt{Lanctot09:Monte} then shows that $\tilde{R}_p^{TK}$ and $R_p^{TK}$ are similar with high probability, leading to

\begin{equation}
    \mathbb{E}\left[\left(\sum_{I\in \mathcal{I}_p}(R_p^{TK}(I) - \tilde{R}_p^{TK}(I))\right)^2\right] \leq
    \frac{2|\mathcal{I}_p||\mathcal{B}_p||A| T K \Delta^2}{\delta^2} 
\end{equation}

(\citealt{lanctot2013monte}, Eq. (4.33), substituting $T \to TK$).

Therefore, by Markov's inequality, with probability at least $1-\rho$,

\begin{equation}
    R_p^{TK} \leq
    \frac{\sqrt{2|\mathcal{I}_p||\mathcal{B}_p||A|TK}\Delta}{\delta\sqrt{\rho}} +
    \frac{\Delta M\sqrt{|A|T}K}{\delta}
\end{equation}
, where external sampling permits $\delta=1$ \cite{lanctot2013monte}.

Using the fact that $M \leq |\mathcal{I}_p|$ and $|\mathcal{B}_p|<|\mathcal{I}_p|$ and dividing through by $KT$ leads to the simplified form

\begin{equation}
    \bar{R}_p^{TK} \leq \left(1 + \frac{\sqrt{2}}{\sqrt{\rho K}}\right)\Delta |\mathcal{I}_p| \frac{\sqrt{|A|}}{\sqrt{T}}
\end{equation}
with probability $1-\rho$.

\end{proof}

We point out that the convergence of $K$-ES is faster as $K$ increases (up to a point), but it still requires the same order of iterations as ES.

\subsection{Proof of Theorem \ref{th:deepcfr_approx}}
\label{app:proof_approx}
\begin{proof}

Assume that an online learning scheme plays

\begin{align}
\label{eq:rm3}
\sigma^t(I,a) = 
\begin{cases}
\frac{y_+^t(I,a)}{\sum_a{y_+^t(I,a)}}\textrm{\ if\ } \sum_a {y_+^t(I,a)} > 0 \\
\textrm{\ arbitrary,\ otherwise}
\end{cases}.
\end{align}

\citealt{Morrill16:Using}, Corollary 3.0.6 provides the following bound on the total regret as a function of the L2 distance between $y_t^+$ and $R^{T,+}$ at each infoset.

\begin{align}
    \label{eq:sigma_y}
    \max_{a\in A} (R^T(I,a))^2 & \leq |A| \Delta^2 T  + 4 \Delta |A| \sum_{t=1}^T  \sum_{a\in A} \sqrt{(R_+^t(I,a) - y_+^t(I,a))^2 } \\
    & \leq |A| \Delta^2 T  + 4 \Delta |A| \sum_{t=1}^T  \sum_{a\in A} \sqrt{(R^t(I,a) - y^t(I,a))^2 }
\end{align}
Since $\sigma^t(I, a)$ from Eq. \ref{eq:rm3} is invariant to rescaling across all actions at an infoset, it's also the case that for any $C(I) > 0$ 
\begin{align}
    \label{eq:sigma_y2}
    \max_{a\in A} (R^T(I,a))^2 & \leq |A| \Delta^2 T  + 4 \Delta |A| \sum_{t=1}^T  \sum_{a\in A} \sqrt{(R^t(I,a) - C(I) y^t(I,a))^2 }
\end{align}

Let $x^t(I)$ be an indicator variable that is 1 if $I$ was traversed on iteration $t$. If $I$ was traversed then $\tilde{r}^t(I)$ was stored in $M_{V,p}$, otherwise $\tilde{r}^t(I)=0$. Assume for now that $\mathcal{M}_{V,p}$ is not full, so all sampled regrets are stored in the memory.

Let $\Pi^t(I)$ be the fraction of iterations on which $x^t(I)=1$, and let $$\epsilon^t(I)=\norm{ \mathbb{E}_t \left[ \tilde{r}^t(I) | x^t(I) = 1 \right] - V(I,a|\theta^t) }_2.$$


\newcommand{\sumxt}{\sum_{t'=1}^t x^{t'}(I)}

Inserting canceling factors of $\sumxt$ and setting $C(I) = \sumxt$,\footnote{The careful reader may note that $C(I)=0$ for unvisited infosets, but $\sigma^t(I,a)$ can play an arbitrary strategy at these infosets so it's okay.}
  \begin{align}
    \max_{a\in A} {(\tilde{R}^T(I, a))^2} \leq & |A| \Delta^2 T + 4 \Delta |A| \sum_{t=1}^T {\left( \sumxt \right) \sum_{a\in A} \sqrt{ \left( \frac{\tilde{R}^t(I,a)}{\sumxt} - y^t(I,a)\right)^2 }} \\
   = & |A| \Delta^2 T + 4 \Delta |A| \sum_{t=1}^T {\left( \sumxt \right) \norm{ \mathbb{E}_t \left[ \tilde{r}^t(I) | x^t(I) = 1 \right] - V(I,a|\theta^t) }_2}  \\
    = & |A| \Delta^2 T + 4 \Delta |A| \sum_{t=1}^T {t \Pi^t(I) \epsilon^t(I) } \hspace{0.5cm} \textrm{by definition}\\
    \leq & |A| \Delta^2 T + 4 \Delta |A| T \sum_{t=1}^T {\Pi^t(I) \epsilon^t(I) }\\
\end{align}


The first term of this expression is the same as Theorem \ref{th:blackwell}, while the second term accounts for the approximation error.

In the case of $K$-external sampling, the same derivation as shown in Theorem \ref{th:blackwell} leads to 

\begin{equation}
\label{th:approx_kes}
\max_{a\in A} {(\tilde{R}^T(I, a))^2} \leq |A| \Delta^2 T K^2 + 4 \Delta \sqrt{|A|} T K^2 \sum_{t=1}^T {\Pi^t(I) \epsilon^t(I) }
\end{equation}
in this case. We elide the proof.

\vspace{0.5cm}

The new regret bound in Eq. (\ref{th:approx_kes}) can be plugged into \citealt{lanctot2013monte}, Theorem 3 as we do for Theorem \ref{thm:mccfr_bound}, leading to 

\begin{equation}
    \label{eq:deepcfr_approx}
    \bar{R}_p^T \leq \sum_{I \in \mathcal{I}_p} \left(
    \left( 1 + \frac{\sqrt{2}}{\sqrt{\rho K}} \right)
    \Delta \frac{\sqrt{|A|}}{\sqrt{T}}  + 
    \frac{4}{\sqrt{T}} \sqrt{|A| \Delta \sum_{t=1}^T \Pi^t(I) \epsilon^t(I) } \right)
\end{equation}

Simplifying the first term and rearranging, 

\begin{align}
\label{eq:regret_pi_epsilon}
    \bar{R}_p^T & \leq \left( 1 + \frac{\sqrt{2}}{\sqrt{\rho K}} \right) \Delta |\mathcal{I}_p| \frac{\sqrt{|A|}}{\sqrt{T}} +  \frac{4\sqrt{|A| \Delta}}{\sqrt{T}} \sum_{I \in \mathcal{I}_p}  \sqrt{ \sum_{t=1}^T \Pi^t(I) \epsilon^t(I) } \\
    \bar{R}_p^T & \leq \left( 1 + \frac{\sqrt{2}}{\sqrt{\rho K}} \right) \Delta |\mathcal{I}_p| \frac{\sqrt{|A|}}{\sqrt{T}} +  \frac{4\sqrt{|A| \Delta}}{\sqrt{T}} |\mathcal{I}_p| \frac{\sum_{I \in \mathcal{I}_p}}{|\mathcal{I}_p|} \sqrt{ \sum_{t=1}^T \Pi^t(I) \epsilon^t(I) } \hspace{1cm} \textrm{Adding canceling factors}\\
   & \leq \left( 1 + \frac{\sqrt{2}}{\sqrt{\rho K}} \right) \Delta |\mathcal{I}_p| \frac{\sqrt{|A|}}{\sqrt{T}} +  \frac{4\sqrt{|A| \Delta |\mathcal{I}_p|}}{\sqrt{T}}  \sqrt{ \sum_{t=1}^T \sum_{I \in \mathcal{I}_p}  \Pi^t(I) \epsilon^t(I) } \hspace{1cm} \textrm{by Jensen's inequality}
\end{align}

Now, lets consider the average MSE loss $\mathcal{L}_V^T(\mathcal{M}^T)$ at time $T$ over the samples in memory $\mathcal{M^T}$.

We start by stating two well-known lemmas:
\begin{lemma} The MSE can be decomposed into bias and variance components
\label{lemma:mse1}
\begin{equation}
\mathbb{E}_x[(x - \theta)^2] = (\theta - \mathbb{E}[x])^2 + \mathrm{Var}(\theta)
\end{equation}
\end{lemma}

\begin{lemma} The mean of a random variable minimizes the MSE loss
\label{lemma:mse2}
\begin{equation}
\argmin_\theta \mathbb{E}_x[(x - \theta)^2] = \mathbb{E}[x]
\end{equation}
and the value of the loss at when $\theta = \mathbb{E}[x]$ is $\mathrm{Var}(x)$.
\end{lemma}

\begin{align}
    \mathcal{L}_V^T & = \frac{1}{ \sum_{I \in \mathcal{I}_p} \sum_{t=1}^T x^t(I) } \sum_{I \in \mathcal{I}_p}  \sum_{t=1}^T x^t(I) \norm{ \tilde{r}^t(I) - V(I|\theta^T)}_2^2 \\
    & \geq  \frac{1}{|\mathcal{I}_p| T} \sum_{I \in \mathcal{I}_p}  \sum_{t=1}^T x^t(I) \norm{ \tilde{r}^t(I) - V(I|\theta^T)}_2^2\\
    & = \frac{1}{|\mathcal{I}_p|} \sum_{I \in \mathcal{I}_p} \Pi^T(I)\ \mathbb{E}_t\left[ \norm{\tilde{r}^t(I) - V(I|\theta^T)}_2^2 \middle| x^t(I)=1 \right]
\end{align}

Let $V^*$ be the model that minimizes $\mathcal{L}^T$ on $\mathcal{M}_T$. Using Lemmas \ref{lemma:mse1} and \ref{lemma:mse2},

\begin{align}
    \mathcal{L}_V^T & \geq \frac{1}{|\mathcal{I}_p| T} \sum_{I \in \mathcal{I}_p} \Pi^T(I)\ \left( \norm{ V(I|\theta^T) - \mathbb{E}_t\left[\tilde{r}^t(I)\middle| x^t(I)=1\right]}_2^2 + \mathcal{L}^T_{V^*} \right)
\end{align}

So, 

\begin{align}
    \mathcal{L}_V^T - \mathcal{L}_{V^*}^T  \geq \frac{1}{|\mathcal{I}_p|} \sum_{I \in \mathcal{I}_p} \Pi^T(I)\ \epsilon^T(I) \\
    \sum_{I \in \mathcal{I}_p} \Pi^T(I)\ \epsilon^T(I) \leq |\mathcal{I}_p| (\mathcal{L}_V^T - \mathcal{L}_{V^*}^T)
\end{align}

Plugging this into Eq. \ref{eq:regret_pi_epsilon}, we arrive at

\begin{align}
    \bar{R}_p^T & \leq \left( 1 + \frac{\sqrt{2}}{\sqrt{\rho K}} \right) \Delta |\mathcal{I}_p| \frac{\sqrt{|A|}}{\sqrt{T}} +  \frac{4\sqrt{|A| \Delta |\mathcal{I}_p|}}{\sqrt{T}}  \sqrt{ |\mathcal{I}_p| \sum_{t=1}^T  (\mathcal{L}_V^t - \mathcal{L}_{V^*}^t) } \\
    & \leq \left( 1 + \frac{\sqrt{2}}{\sqrt{\rho K}} \right) \Delta |\mathcal{I}_p| \frac{\sqrt{|A|}}{\sqrt{T}} +  4 |\mathcal{I}_p| \sqrt{|A| \Delta \epsilon_\mathcal{L}}
\end{align}

So far we have assumed that $\mathcal{M}_V$ contains all sampled regrets. The number of samples in the memory at iteration $t$ is bounded by $K \cdot |\mathcal{I}_p| \cdot t$. Therefore, if $K \cdot |\mathcal{I}_p| \cdot T < |\mathcal{M}_V|$ then the memory will never be full, and we can make this assumption.\footnote{We do not formally handle the case where the memories become full in this work. Intuitively, reservoir sampling should work well because it keeps an `unbiased' sample of previous iterations' regrets. We observe empirically in Figure~\ref{fig:ablations} that reservoir sampling performs well while using a sliding window does not.}
\end{proof}

\subsection{Proof of Corollary \ref{cor:deepcfr_approx}}

\begin{proof}
Let $\rho = T^{-1/4}$. 

\begin{equation}
       P\left(\bar{R}_p^T > \left( 1 + \frac{\sqrt{2}}{\sqrt{K}} \right) \Delta |\mathcal{I}_p| \frac{\sqrt{|A|}}{T^{-1/4}} +  4 |\mathcal{I}_p| \sqrt{|A| \Delta \epsilon_\mathcal{L}} \right) < T^{-1/4}
\end{equation}

Therefore, for any $\epsilon > 0$,

\begin{equation}
       \lim_{T\to \infty} P\left(\bar{R}_p^T - 4 |\mathcal{I}_p| \sqrt{|A| \Delta \epsilon_\mathcal{L}} >  \epsilon \right) = 0.
\end{equation}
\end{proof}

\newpage

\section{Network Architecture}

\label{app:network}
In order to clarify the network architecture used in this work, we provide a PyTorch \cite{Paszke17:Automatic} implementation  below.

\small
\begin{lstlisting}[language=Python]
import torch
import torch.nn as nn
import torch.nn.functional as F

class CardEmbedding(nn.Module):
    def __init__(self, dim):
        super(CardEmbedding, self).__init__()
        self.rank = nn.Embedding(13, dim)
        self.suit = nn.Embedding(4, dim)
        self.card = nn.Embedding(52, dim)

    def forward(self, input):
        B, num_cards = input.shape
        x = input.view(-1)

        valid = x.ge(0).float()  # -1 means 'no card'
        x = x.clamp(min=0)

        embs = self.card(x) + self.rank(x // 4) + self.suit(x % 4)
        embs = embs * valid.unsqueeze(1)  # zero out 'no card' embeddings

        # sum across the cards in the hole/board
        return embs.view(B, num_cards, -1).sum(1)

class DeepCFRModel(nn.Module):
    def __init__(self, ncardtypes, nbets, nactions, dim=256):
        super(DeepCFRModel, self).__init__()

        self.card_embeddings = nn.ModuleList(
            [CardEmbedding(dim) for _ in range(ncardtypes)])

        self.card1 = nn.Linear(dim * ncardtypes, dim)
        self.card2 = nn.Linear(dim, dim)
        self.card3 = nn.Linear(dim, dim)

        self.bet1 = nn.Linear(nbets * 2, dim)
        self.bet2 = nn.Linear(dim, dim)

        self.comb1 = nn.Linear(2 * dim, dim)
        self.comb2 = nn.Linear(dim, dim)
        self.comb3 = nn.Linear(dim, dim)

        self.action_head = nn.Linear(dim, nactions)

    def forward(self, cards, bets):

        """
        cards: ((N x 2), (N x 3)[, (N x 1), (N x 1)])  # (hole, board, [turn, river])
        bets: N x nbet_feats
        """

        # 1. card branch
        # embed hole, flop, and optionally turn and river
        card_embs = []
        for embedding, card_group in zip(self.card_embeddings, cards):
            card_embs.append(embedding(card_group))
        card_embs = torch.cat(card_embs, dim=1)

        x = F.relu(self.card1(card_embs))
        x = F.relu(self.card2(x))
        x = F.relu(self.card3(x))

        # 1. bet branch
        bet_size = bets.clamp(0, 1e6)
        bet_occurred = bets.ge(0)
        bet_feats = torch.cat([bet_size, bet_occurred.float()], dim=1)
        y = F.relu(self.bet1(bet_feats))
        y = F.relu(self.bet2(y) + y)

        # 3. combined trunk
        z = torch.cat([x, y], dim=1)
        z = F.relu(self.comb1(z))
        z = F.relu(self.comb2(z) + z)
        z = F.relu(self.comb3(z) + z)

        z = normalize(z)  # (z - mean) / std
        return self.action_head(z)
\end{lstlisting}

\end{document}